\relax
\documentclass[letterpaper]{article} 
\usepackage{aaai21}  
\usepackage{times}  
\usepackage{helvet} 
\usepackage{courier}  
\usepackage[hyphens]{url}  
\usepackage{graphicx} 
\usepackage{graphicx}  
\usepackage{natbib}  
\usepackage{caption} 
\usepackage[switch]{lineno}  %
\usepackage{amsmath, amssymb, amsfonts, amsthm}
\usepackage{mathtools}
\usepackage[ruled,vlined]{algorithm2e}
\SetKwInput{KwInput}{Input}
\SetKwInput{KwOutput}{Output}
\usepackage{tikz-cd}
\usepackage{booktabs}
\usepackage{multirow}
\usepackage{bm}
\usepackage[multiple]{footmisc}

\newtheorem{theorem}{Theorem}
\newtheorem{lemma}{Lemma}

\newtheorem{proposition}{Proposition}

\theoremstyle{definition}

\newtheorem{definition}{Definition}

\newcommand{\real}{\mathbb{R}}
\newcommand{\ball}{\mathbb{B}}
\newcommand{\expect}{\mathbb{E}}

\DeclareMathOperator*{\argmax}{arg\,max}
\DeclareMathOperator*{\argmin}{arg\,min}

\DeclareUnicodeCharacter{0301}{\'{e}}
\DeclareUnicodeCharacter{0300}{\'}

\begin{document}
\title{Deterministic Certification to Adversarial Attacks via\\ Bernstein Polynomial Approximation}
\author{Ching-Chia Kao\thanks{Equal contribution}$^{1}$, Jhe-Bang Ko\footnotemark[1]$^{1}$, Chun-Shien Lu$^{1}$\\}
\affiliations{
    $^{1}$Institute of Information Science, Academia Sinica, Taiwan\\
    \{cck123, b05505053, lcs\}@iis.sinica.edu.tw\\
}

\maketitle

\begin{abstract}
    \begin{quote}
        \emph{Randomized smoothing} has established state-of-the-art provable robustness against $\ell_2$ norm adversarial attacks with high probability. However, the introduced Gaussian data augmentation causes a severe decrease in natural accuracy. We come up with a question, ``Is it possible to construct a smoothed classifier without randomization while maintaining natural accuracy?''. We find the answer is definitely yes. We study how to transform any classifier into a certified robust classifier based on a popular and elegant mathematical tool, Bernstein polynomial. Our method provides a deterministic algorithm for decision boundary smoothing. We also introduce a distinctive approach of norm-independent certified robustness via numerical solutions of nonlinear systems of equations. Theoretical analyses and experimental results indicate that our method is promising for classifier smoothing and robustness certification.
    \end{quote}
\end{abstract}

\section{Introduction}
Neural Network models achieve an enormous breakthrough in image classification these years but are vulnerable to imperceptible perturbations known as adversarial attacks, as evidenced by  \cite{GoodfellowSS14} \cite{Szegedy13}. Since then, many researchers began to develop their adversarial attacks such as DeepFool \cite{moosavi2016deepfool}, Jacobian based Saliency Map Attack (JSMA) \cite{papernot2016limitations}, CW attack \cite{carlini2017towards}, etc. On the other hand, researchers also attempted to build defense mechanisms that are invincible to adversarial attacks. Please see \cite{yuan2019adversarial} for the thorough surveys of adversarial attacks and defense approaches. However, we still have a long way to go in that making an indestructible machine learning model is still an open problem in the AI community.

Empirical defense strategies can be categorized into several groups. The first one is adversarial detection \cite{feinman2017detecting} \cite{ma2018characterizing}. However, \cite{carlini2017adversarial} bypassed ten detection methods and showed that it is extremely challenging to detect adversarial attacks. The second one is denoising \cite{liao2018defense} \cite{xie2019feature} \cite{xu2017feature}. The kind of methods is computationally inefficient (some works trained on hundreds of GPUs) but useful in practice. The third one is called adversarial training, which is the most effective but suffers from attack dependency \cite{tramer2018ensemble} \cite{madry2018towards}. In addition to the aforementioned methods, there are still various defense strategies proposed. \cite{pmlr-v80-athalye18a} showed that most of the defense methods are ineffective when encountering sophisticated designed attacks due to lack of a provable defense guarantee. 
\begin{figure}[!t]
  \centering
    \includegraphics[width=.4\textwidth]{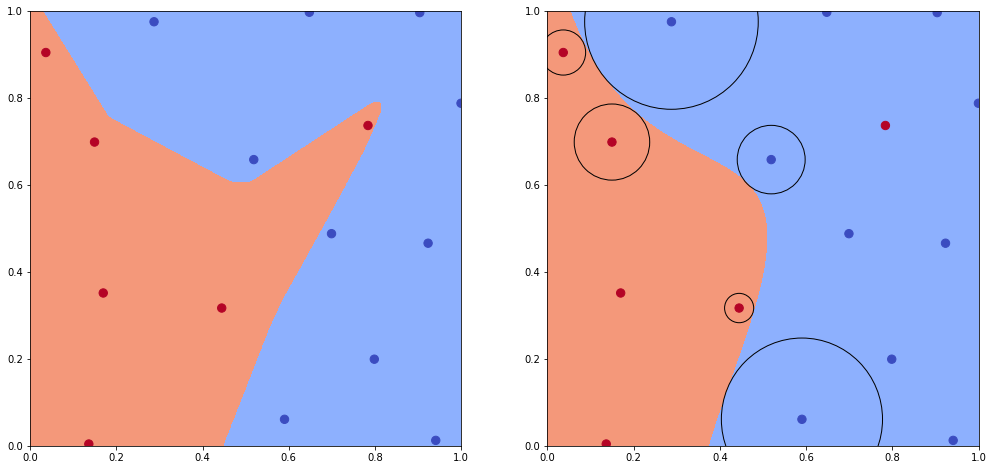}
    \caption{A 2D example to visualize our result. {\bf Left:}           
     The classifier's decision boundary from the neural network is very ``sharp". Points
     near the sharp region are vulnerable to adversarial attacks. {\bf Right:} We can see 
     that the boundary of the classifier from our method is smoothed, and the 
     certification (``safe zone") is mostly accurate.}
    \label{fig:compare}
\end{figure}
To avoid this arms race, a series of researches \cite{salman2019convex} on provable defenses or certified robustness have emerged with theoretical guarantees. Specifically, for any point $x$, there is a set containing the provably robust $x$, implying that every point in this set gives the same prediction.

Assume that adversarial examples are caused due to the irregularity of the decision boundary. \emph{Randomized smoothing} \cite{cohen2019certified} introduced Gaussian random noise for smoothing the base classifier. Nonetheless, many challenges remain. Firstly, Gaussian data augmentation causes a severe decrease in natural accuracy. Secondly, the prediction and certification require numerous samples during inference, which is a time-consuming process. Thirdly, there is still a chance, even slightly, that the certification is inaccurate. In many applications, like self-driving vehicles, we do not want to take this kind of risk. 

In this paper, we propose a new method to ``smooth" the decision boundary of the classifier. Our method can also provide a ``safe zone" that predicts the same class as the input deterministically. We use a 2D example to demonstrate our method in Figure~\ref{fig:compare}. Note that the regions near the decision boundary are more comfortable to certify by our approach, while the areas far from the decision boundary are less accurate. It is due to the initial guess of the solver of nonlinear systems of equations. To overcome this problem, we set the error tolerance of the solution to an acceptable threshold \cite{2020SciPy-NMeth}. Furthermore, we care more about data near the decision boundary than those far away from it since we can find adversarial examples with our bare eyes if the attack is strong enough.

Our contributions are summarized as follows:
\begin{itemize}
    \item To our knowledge, we are the first to introduce Bernstein polynomial into the adversarial example community for which we devise a deterministic algorithm to smooth the base classifier. 
    \item Our method, compared with \cite{cohen2019certified} and \cite{salman2019provably}, can maintain higher natural accuracy while achieving comparable robustness (certified accuracy) in CIFAR-10\footnote{Other results obtained from different datasets such as MNIST, SVHN, and CIFAR-100 would be discussed in the Appendix} \cite{krizhevsky2009learning}.
    \item We can certify a smoothed classifier with arbitrary norms (norm independent).
    \item Our method can be used in different aspects of machine learning beyond certified robustness, such as over-fitting alleviation. 
\end{itemize}

\section{Related Works}
We roughly categorize adversarial defense researches into empirical defenses and certified defenses. Empirical defenses seem to be robust to existing adversarial perturbations but lack of formal robustness guarantee. In this section, we will study previous empirical defenses and certified defenses with the focus on \emph{randomized smoothing}.

\subsection{Empirical defenses}
In practice, the best empirical defense is Adversarial Training (AT) \cite{GoodfellowSS14} \cite{KurakinGB17} \cite{madry2018towards}. AT is operated by first generating the adversarial examples by projected gradient descent and then augmenting them into the training set. Although the classifier yielded based on AT is robust to most gradient-based attacks, it is still difficult to tell whether this classifier is undoubtedly robust. Moreover, adaptive attacks break most empirical defenses. A pioneering study \cite{pmlr-v80-athalye18a} showed that most of the methods are ineffective by using sophisticated designed attacks. To stop this arms race between attackers and defenders, some researchers try to focus on building a verification mechanism with a robustness guarantee. 

\subsection{Certified defenses}
Certified defenses can be divided into {\emph exact} (a.k.a ``complete") methods and {\emph conservative} (a.k.a ``incomplete") methods. Exact methods are usually based on Satisfiability Modulo Theories solvers \cite{katz2017reluplex} \cite{katz2017towards} or mixed-integer linear programming \cite{tjeng2018evaluating}, but they are computationally inefficient. Conservative methods are guaranteed to find adversarial examples if they exist but might mistakenly judge a safe data point as an adversarial one (false positive) \cite{pmlr-v80-wong18a} \cite{wong2018scaling} \cite{salman2019convex} \cite{wang2018efficient} \cite{wang2018mixtrain}  \cite{dvijotham2018dual} \cite{raghunathan2018certified} \cite{raghunathan2018semidefinite} \cite{croce2019provable}\cite{gehr2018ai2} \cite{mirman2018differentiable} \cite{singh2018fast} \cite{pmlr-v80-weng18a} \cite{gowal2018effectiveness} \cite{zhang2018efficient}. 

\subsection{\emph{Randomized Smoothing}}

Introducing randomness into neural network classifiers has been used as a heuristic defense method against adversarial perturbation. Some works  \cite{liu2018towards} \cite{cao2017mitigating} introduced randomness without provable guarantees. \cite{lecuyer2019certified} first used inequalities from differential privacy to prove robustness guarantees of $\ell_2$ and $\ell_1$ norm with Gaussian and Laplacian noises, respectively. Later on, \cite{li2019certified} used information theory to prove a stronger $\ell_2$ robustness guarantee for Gaussian noise. All these robustness guarantees are still loose. \cite{cohen2019certified} provided a tight robustness guarantee for \emph{randomized smoothing}.

\cite{salman2019provably} integrated \emph{randomized smoothing} with adversarial training to improve the performance of smoothed classifiers. \cite{Zhai2020MACER} \cite{FengWCZN20} both considered natural and robust errors in their loss function. The authors tried to maximize the certified radius during training. These methods have great success using \emph{randomized smoothing} but some researchers \cite{Ghiasi2020BREAKING} \cite{kumar2020curse} \cite{blum2020random} raised concerns about \emph{randomized smoothing}.

\cite{kumar2020curse} \cite{blum2020random} claimed that extending the smoothing technique to defend against other attacks can be challenging, especially in the high-dimensional regime. The largest $\ell_p$-radius that can be certified decreases as $O(1/d^{\frac{1}{2} - \frac{1}{p}})$ with dimension $d$ for $p > 2$. This means that the radius decreases due to the curse of dimensionality.

\section{Problem Setup}
In this section, we will set up our problem mathematically to make this paper self-contained. The notations frequently used in this paper are described in Table~\ref{tab:notation} in the Appendix. Here, we consider the conventional classification problem. Suppose that $\mathcal{X} \subseteq \real^m$ is the set of data and $\mathcal{Y} \subseteq \real^K$ is the set of labels.
We can train a (base) classifier $f : \mathcal{X} \rightarrow \mathcal{Y}$.
If for each $x$ one can find a $\|\delta\|_p \leq \epsilon$ such that $\argmax_{i\in[K]}f_i(x+\delta) \neq \argmax_{i\in[K]}f_i(x)$ , where $\|\cdot\|_p$ denotes $p$-norm for $p > 1$ and $[K] = \{1,2,\cdots,K\}$. We call $x_{adv} = x + \delta$ adversarial example. 

The following are our goals:
\begin{enumerate}
    \item To find a transformation $\tau : C(\mathcal{X};\real^K) \rightarrow C^\infty(\mathcal{X};\real^K)$, with the purpose of translating a classifier into a smoothed classifier, denoted as $\Tilde{f} = \tau(f)$.
    \item To find a minimal $R$ with respect to $p$-norm,  where $p>1$, such that $\argmax_{i\in[K]}\Tilde{f_i}(x+R) \neq \argmax_{i\in[K]}\Tilde{f_i}(x)$. We call $R$ the certified radius of $x$.
\end{enumerate}

\begin{figure}[!t]
    \centering
    \includegraphics[width=.45\textwidth]{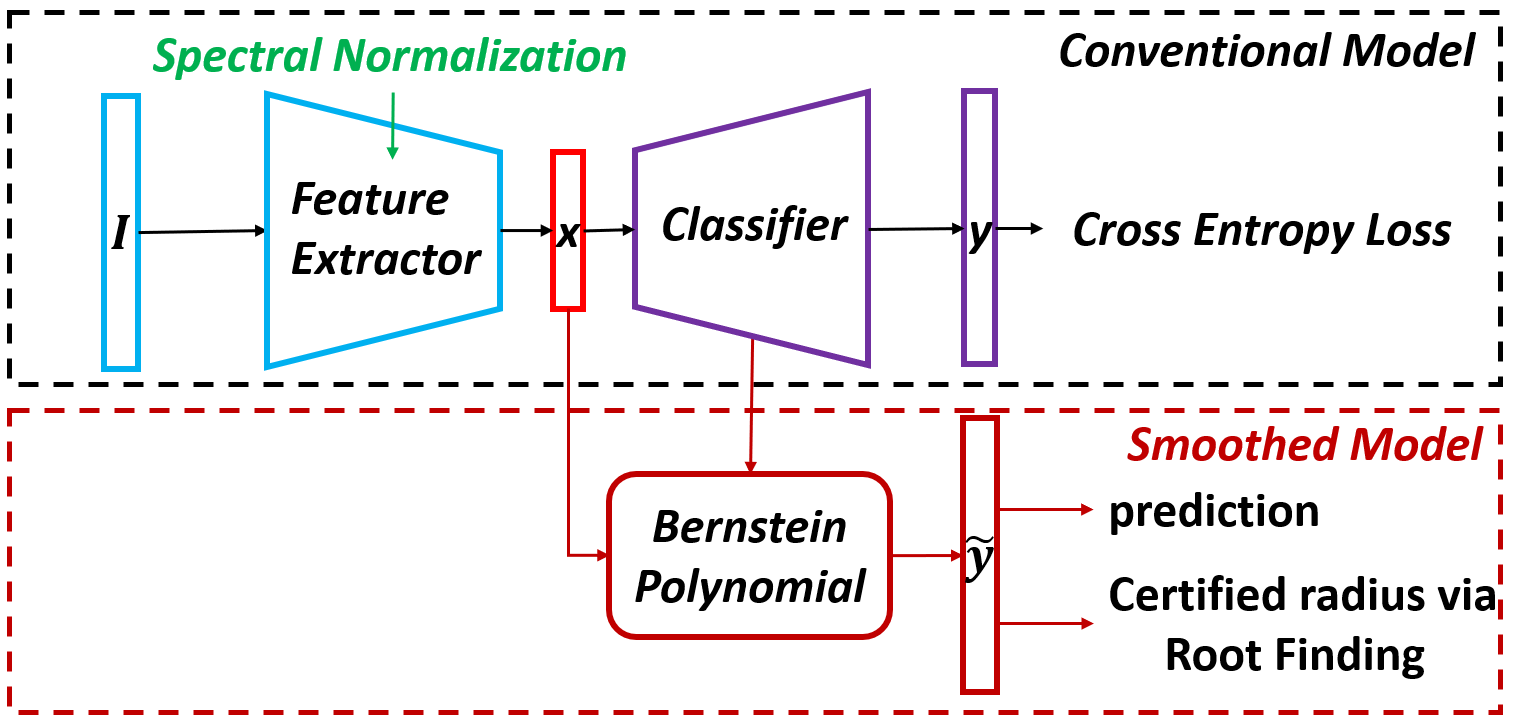}
    \caption{Architecture Design. {\bf Top:} A typical image classifier specialized in spectral normalization, adding to each layer of a feature extractor. We compress the layer $x$ to a small dimension. {\bf Bottom:} We adopt Bernstein polynomial to smooth the classifier for prediction and certification.} 
    \label{fig:Arch}
\end{figure}

\section{Proposed Method}
In this paper, we concern about ``What is a smooth classifier?''. We can think it as a function that can be differentiated infinitely times, and the function we differentiate is still continuous. The simplest example we can think of is polynomials. Hence, we try to approximate a neural network classifier via Bernstein polynomial. Unlike convolution with a Gaussian distribution or other distributions, the exact computation of Bernstein polynomial is more natural to apply.

Our method contains three parts, as illustrated in Figure~\ref{fig:Arch}.
The first part is the spectral normalization. This mechanism is usually introduced to stably train the discriminator in Generative Adversarial Networks (GANs). Spectral normalization is adopted here because it can make the feature extractor's output be bounded due to the change of input. 

The second part is the Bernstein polynomial approximation, which is the most important one. Bernstein polynomial takes a feature vector and the classifier into consideration (See Equation~(\ref{eqn:d-dim-bpoly})). We can also consider this form as uniform sampling in $[0,1]^d$ and place these samples into the base classifier to obtain the coefficients. With these coefficients, we manage to use a linear combination of Bernstein basis to construct a smoother version of the base classifier. To overcome the computational complexity of high dimensional Bernstein polynomials, we compress $x$ to have small (say 3$\sim$6) dimensions. 

The last part is the process of calculating the certified radius via root-finding. We construct nonlinear systems of equations to characterize the points on the decision boundary and solve them by numerical root-finding techniques. After applying this process, we obtain a certified radius for each input data.

\subsection{Spectral Normalization (SN)}
According to \cite{miyato2018spectral}, spectral normalization controls the Lipschitz constant of neural networks that were originally used in stabilizing the training of the discriminator in GANs. The formulation is stated as:

\begin{align}\label{eqn:SN}
    W_{SN}^{(i)} = \frac{W_i}{\sigma(W_i)},\sigma(W_i) = \max_{x:x\neq0}\frac{\|W_ix\|_2}{\|x\|_2},
\end{align}
which is equivalent to being divided by the greatest singular value of $W_i$, the weights of $i$-th layer. Suppose that our feature extractor $G$ (see Figure 2) is the composition function composed of Lipschitz continuous functions $g_i(x^{(i)};W_{SN}^{(i)}):= ReLU(W_{SN}^{(i)}x^{(i)} + b^{(i)})$, where $x^{(i)}$ is the output of $g_{i-1}$ and $b^{(i)}$ is the bias for $i=1,2,\cdots,M$, $i.e.$, 

\begin{align}\label{eqn:lips1}
    G(x^{(1)}) := g_M(g_{M-1}\cdots g_1(x^{(1)};W_{SN}^{(1)});W_{SN}^{(2)})\cdots;W_{SN}^{(M)}),
\end{align}
where $x^{(1)} = I$ is the input data. For each $g_i$ we have

\begin{align}\label{eqn:lips2}
    \|g_i(x^{(i)}+\delta;W_{SN}^{(i)}) - g_i(x^{(i)};W_{SN}^{(i)})\|_2 \leq \|W_{SN}^{(i)}\|_2\|\delta\|_2.
\end{align}
Given two input data $I$ and $I'$, and by Equation~(\ref{eqn:SN}), we have $\|W_{SN}^{(i)}\|_2 = 1$ for $i = 1,2,\cdots,M$. We can derive that the Lipschitz constant of $G$ is one by Equation~(\ref{eqn:lips1}) and (\ref{eqn:lips2}) as:
\begin{center}
    If $G(I) \neq G(I')$, there exists an $R>0$ such that
\end{center}
\begin{align}\label{eqn:lips}
    R \leq \|G(I) - G(I')\|_2 \leq \|I-I'\|_2.
\end{align}
We can further observe from Equation~(\ref{eqn:lips}) that the upper bound of the certified radius is tighter in feature domain than that in the image domain.



\subsection{Smoothing the Classifier via Bernstein Polynomial}

Weierstrass approximation theorem asserts that a continuous function on a closed and bounded interval can be uniformly approximated on that interval by polynomials to any degree of accuracy. All of the formal definitions are described as follows \cite{davis1975interpolation}.

\begin{theorem}[{\bf Weierstrass Approximation Theorem}]\label{Weierstrass}
  Let $f \in C([a,b])$. Given $\epsilon >0$, there exists a polynomial $p(x)$ such that 
  
  \begin{equation*}
      |f(x) - p(x)| \leq \epsilon, \hspace{2ex} \forall x\in [a,b].
  \end{equation*}
\end{theorem}

Bernstein proved the Weierstrass Approximation theorem in a probability approach \cite{bernstein1912}. The rationale behind Bernstein polynomial is to compute the expected value $\expect[f(\eta/n)]$, where $\eta \sim Bin(n,x)$ and $Bin(n,x)$ is the binomial distribution. Note that the parameter $n$, relating to how it can be used to smooth the classifier, will be elaborated in detail later. The following Definition~\ref{bpoly} and Theorem~\ref{bernstein} are from \cite{bernstein1912}.

\begin{definition}[{\bf Bernstein Polynomial}]\label{bpoly}
  Let $f(x)$ be a function defined on $[0,1]$. The $n$-th Bernstein Polynomial is defined by
  
  \begin{equation}\label{eqn:bpoly}
      B_n(f;x) = \sum\limits_{k=0}^nf\left(\frac{k}{n}\right)\binom{n}{k}x^k(1-x)^{n-k}.
  \end{equation}
  Note that 
  
  \begin{equation}
      B_n(f;0) = f(0), \hspace{2ex} B_n(f;1) = f(1). \notag
  \end{equation}
\end{definition}


\begin{theorem}[{\bf Bernstein}]\label{bernstein}
  Let $f(x)$ be bounded on $[0,1]$. Then
  
  \begin{equation}
      \lim_{n \rightarrow \infty}B_n(f;x) = f(x)
  \end{equation}
  at any point $x\in[0,1]$ at which $f$ is continuous. If $f \in C([0,1])$, the limit holds uniformly in $[0,1]$.
\end{theorem}
Back to Theorem~\ref{Weierstrass}, we can take it as a corollary of Theorem~\ref{bernstein}.
Since the neural network is a multi-dimensional function, we have to generalize the one dimensional Bernstein polynomial (Definition~\ref{bpoly}) into a multi-dimensional version. The following is one kind of generalization which considers independent random variables expectation \cite{veretennikov2016partial}. 

\begin{definition}[{\bf $\bm{d}$-dimensional Bernstein Polynomial}]\label{def:bpoly-multi}
  Let $f: [0,1]^d \rightarrow \real$. We can define $d$-dimensional Bernstein Polynomial by
  
  \begin{equation}
    \begin{split}
        &B_{n_1,\cdots,n_d}(f;x_1,\cdots,x_d)\\
       &= \sum\limits_{\substack{0\leq k_j \leq n_j \\ j \in \{1,\cdots,d\}}}f\left(\frac{k_1}{n_1},\cdots,\frac{k_d}{n_d}\right)\prod\limits_{j=1}^d \binom{n_j}{k_j}{x_j}^{k_j}(1-x_j)^{n_j - k_j}.\notag
    \end{split}
  \end{equation}
  For simplicity, we abbreviate the notation by setting $\frac{k}{n} := \left(\frac{k_1}{n_1},\cdots,\frac{k_d}{n_d}\right)$ and $b_j := \binom{n_j}{k_j}{x_j}^{k_j}(1-x_j)^{n_j - k_j}$, and have
  
  \begin{equation}\label{eqn:d-dim-bpoly}
       B_n(f;x) = \sum\limits_{\substack{0\leq k_j \leq n_j \\ j \in \{1,\cdots,d\}}} f\left(\frac{k}{n}\right)\prod\limits_{j=1}^d b_j.
  \end{equation}
\end{definition}

\subsubsection{How to Smooth the Classifier via Bernstein Polynomial?}
As we can see from Theorem 2, as $n$ becomes larger, the Bernstein polynomial will gradually approximate the original function. Thus, the rationale behind our Bernstein polynomial-based smoothed classifier is that we can choose a proper $n$ to smooth and approximate the base classifier. Obviously, there is a trade-off between smoothness (robustness) and accuracy (natural accuracy). We use an example to visualize Theorem~\ref{bernstein} in Figure~\ref{fig:toy_example_trials}.

\begin{figure}[!htbp]
    \centering
    \includegraphics[width=.4\textwidth]{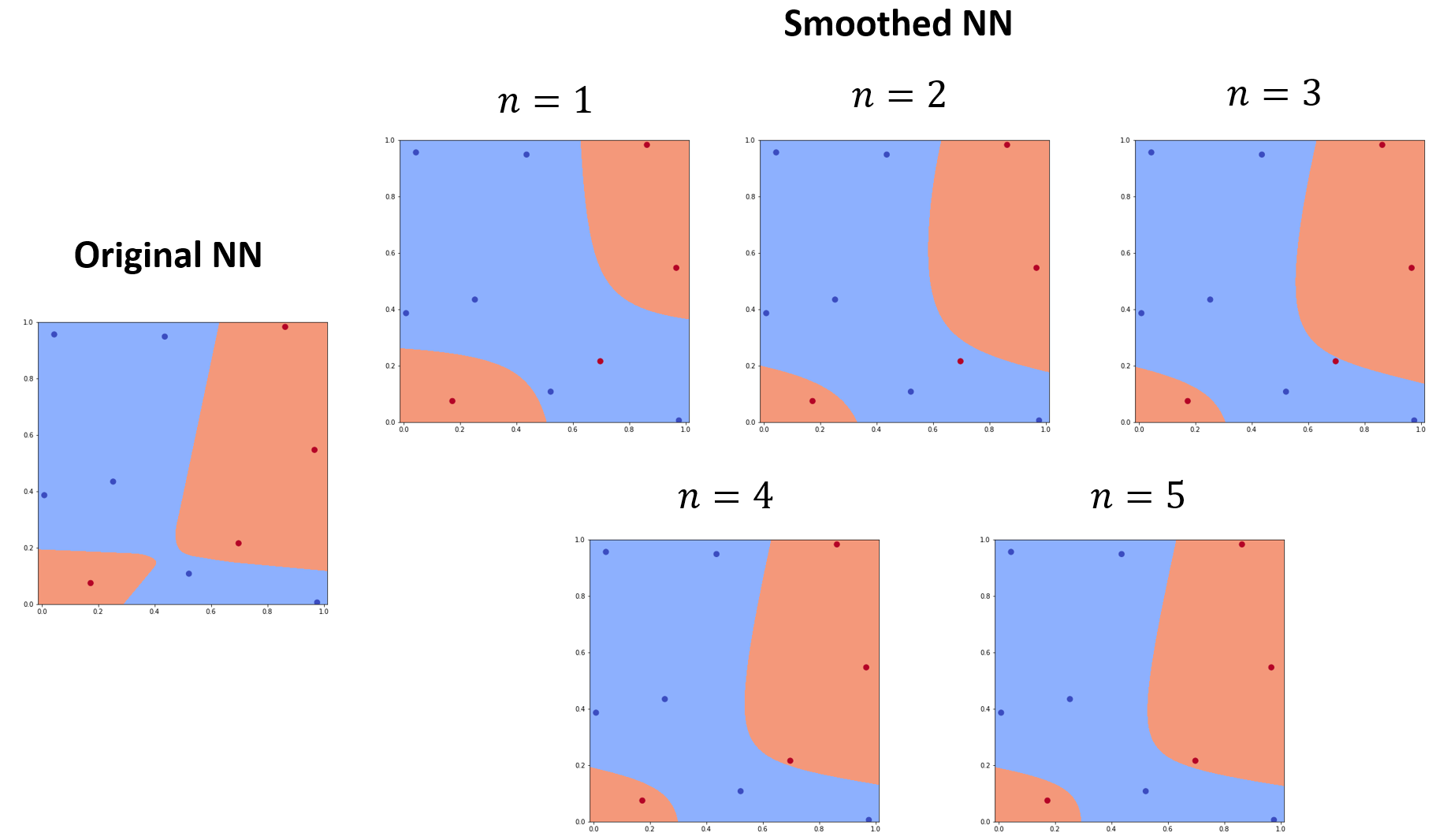}
    \caption{Decision boundaries of the neural network variants under different $n$s. As we can see, the smoothed decision boundary obtained with a larger $n$ gradually approximates the original neural network.}
    \label{fig:toy_example_trials}
\end{figure}
By Definition~\ref{def:bpoly-multi}, we can transform our classifier $f:[0,1]^d\rightarrow\real^K$, $f = [f_1,f_2,\cdots, f_K]$, into a smoothed classifier as

\begin{align}
    \Tilde{f}(x) = [B_n(f_1;x),\cdots,B_n(f_K;x)],
\end{align}
with a proper choice of $n$.

Without any prior knowledge of the pre-trained neural network, we set $n_1 = n_2 = \cdots = n_d = n$, which denotes the uniform sampling in all directions. We can see that the complexity of computing $d$-dimensional Bernstein Polynomial is $O(n^d)$, which is too high in the image domain even if we take $n=2$. To overcome this problem, we propose to exploit dimension reduction in the feature layer. As we will see later in our empirical results, less than $1\%$ natural accuracy will be sacrificed in MNIST \cite{lecun2010mnist} and CIFAR-10 due to dimensionality reduction in the feature space. 
\subsection{Certified Radius via Root-Finding}
To find the certified radius of the input data, different from \cite{cohen2019certified}, we are not trying to find a closed form solution. Instead, we propose to use numerical root-finding techniques to find an adversarial example in the feature domain. 

We abbreviate the notation by setting $\beta_i = B_n(f_{i};x)$ and $S(\beta)_i = \exp(\beta_i)/\sum_{j = 1}^{K}\exp(\beta_j)$, where $K$ is the number of classes. Note that $S(\beta)$ is the traditional softmax function which represents the probability vector. Also, recall that $d \leq K$ is the dimension of a feature vector $x_0$.

In the worst case analysis, the runner-up class is the easiest target for attackers. Hence, we assume that the closest point to a feature vector $x_0$ on the decision boundary between the top two predictions of the smoothed classifier $\Tilde{f}(x) = \beta = [\beta_1,\cdots,\beta_K]$ follows $\beta_{\rho(1)} = \beta_{\rho(2)}$, where $\rho$ is the mapping from ranks to predictions. Therefore, the nonlinear systems of equations for characterizing the points on the decision boundary of smoothed classifier $\Tilde{f}(x)$ are described as:

\begin{equation}
    \begin{cases}
    \phi_0 := \beta_{\rho(1)} - \beta_{\rho(2)} = 0\\
    \phi_1 := S(\beta)_{\rho(1)} - 0.5 = 0\\
    \phi_2 := S(\beta)_{\rho(2)} - 0.5 = 0\\
    \phi_3 := S(\beta)_{\rho(3)} = 0\\
    \hspace{8ex}\vdots\\
    \phi_d := S(\beta)_{\rho(d)} = 0\\ 
    \end{cases}
    \label{eqn:nonlinear}
\end{equation}
where $\phi_0$ can be seen as the decision boundary between the top two predictions, and $\phi_1$ and $\phi_2$ are the probabilities of the top two predictions, respectively. 
The remaining equations indicate that the probabilities other than $\rho(1)$ and $\rho(2)$ are all zeros.
The only thing we need to do is to start from the initial guess (feature vector $x_0$) and find the adversary nearest to it.

To solve Equation~(\ref{eqn:nonlinear}), we transform it into a minimization problem. We consider the function $\Phi: \real^d \rightarrow \real^d$, $\Phi = (\phi_0, \phi_1, \cdots, \phi_d)$, and minimize $\frac{1}{2}\|\Phi(x)\|_2^2$ is equivalently to finding 

\begin{align}\label{eqn:ls}
    x^* = \argmin_{x}\{F(x)\},
\end{align}
where  $F(x)=\frac{1}{2}\sum\limits_{i=0}^{d}(\phi_i(x))^2=\frac{1}{2}\|\Phi(x)\|_2^2.$

Suppose that $x^*$ is the optimal solution to Equation~(\ref{eqn:nonlinear}), that is, $\Phi(x^*) = 0$. According Equation~(\ref{eqn:ls}), we have $F(x^*) = 0$ and $F(x) > 0$ if $\Phi(x) \neq 0$. The details of how to solve Equation~(\ref{eqn:ls}) are described in the Appendix. Let $\emph{sol}$ denote the solution of Equation~(\ref{eqn:ls}), we can easily compute $R = \|x_0 - \emph{sol}\|_p$ for $p>1$, which is known as the certified radius. Remark that this radius is computed in the feature domain (see Equation~(\ref{eqn:lips})). On the other hand, we prove that our method satisfies the following proposition specialized in norm independence.

\begin{proposition}\label{prop:1}
    Let $\Tilde{f}: [0,1]^d \rightarrow \real^K$ be a smoothed function and $A=\{\Tilde{y}\in \real^K : \Tilde{y}_i\neq\Tilde{y}_j, \forall i \neq j\}$ denote the area without the smoothed decision boundary. For any $x \in [0,1]^d$ and $\Tilde{f}(x) \in A$, there exists an $R>0$ such that $\argmax_{i\in[K]}\Tilde{f_i}(x+\hat{\delta})= \argmax_{i\in[K]}\Tilde{f_i}(x)$ whenever $\|\hat{\delta}\|_p < R$ \ for \ $p>1$.
\end{proposition}
The proof can be found in the Appendix.


\subsection{Inference procedure}
Suppose that we have a pre-trained model for the purpose of either prediction or certification. The prediction is a simple substitution of a smoothed classifier for the base classifier. Hence, we mainly illustrate our smoothing and certification procedures in Algorithm~\ref{alg:smooth} and \ref{alg:cert}, respectively.

\begin{algorithm}
\SetAlgoLined
\KwInput{A classifier $f$, a feature $x_0$, and a parameter $n$}
\KwOutput{$\beta(= \Tilde{f}(x_0))$, the output of the smoothed classifier}
 Initialize an empty array $\beta = [ \hspace{1ex}]$\;
 \For{i=1,2,$\cdots$, K}
 {Compute Equation~(\ref{eqn:d-dim-bpoly}) by $f_i, x_0$ and $n$ to get $\beta_i$\;
 Append $\beta_i$ to $\beta$\;
 }
 \Return $\beta$
 \caption{Smooth($f$, $x_0$, $n$)}
 \label{alg:smooth}
\end{algorithm}

\begin{algorithm}
\SetAlgoLined
\KwInput{An input data $I$, a feature extractor $G$, a classifier $f$, and a parameter $n$}
\KwOutput{Certified Radius $R$}
 $x_0 = G(I)$\;
 $\beta$ = Smooth($f$, $x_0$, $n$)\;
 Setup Equation~(\ref{eqn:nonlinear}) by $\beta$\ and the dimension of $x_0$\;
 Let $x_0$ be the initial guess\;
 Transform Equation~(\ref{eqn:nonlinear}) into Equation~(\ref{eqn:ls})\;
 Solve Equation~(\ref{eqn:ls}) by a least-squares solver\;
 Project $\emph{sol}$ onto $[0,1]^d$\;
 $R = \|x_0 - \emph{sol}\|_p$\;
 \Return $R$
 \caption{Certification($I$, $G$, $f$, $n$)}
 \label{alg:cert}
\end{algorithm}
\subsubsection{Deterministic algorithm}
Note that our deterministic algorithm has two-fold meanings, $i.e.$, deterministic smoothing and deterministic certification. By the definition of the deterministic algorithm, given a particular input, it always produces the same output. Hence, Algorithm~\ref{alg:smooth} and \ref{alg:cert} are both deterministic. There are some advantages and disadvantages to the robust classification task. The most significant advantage is that our algorithm does not depend on any random value; it always produces a stable output (robustness). The disadvantage is that for attackers, attacking the smoothed classifier is no more difficult than attacking the non-smoothed classifier since it is end-to-end differentiable. Thus, we do not claim that our method is a defense mechanism but a certification/verification method. 

\section{Experiments}

In this section, we aim to compare our method with state of the art certification methods, \emph{randomized smoothing}  \cite{cohen2019certified} \cite{salman2019provably}, on CIFAR-10. We also compare the empirical upper bound using PGD attack \cite{madry2018towards} with our certified accuracy at all radii. Other results obtained from different datasets such as MNIST, SVHN \cite{netzer2011reading}, and CIFAR-100 \cite{krizhevsky2009learning} would be discussed in the Appendix. In addition, we perform the ablation study. 

\subsection{Important Metrics}

We are interested in three significant metrics, natural accuracy, robust accuracy, and certified accuracy. Natural accuracy is the number of correctly identified clean testing examples divided by the number of total testing examples. Robust accuracy is the number of correctly identified adversarial testing examples divided by the number of total testing examples. Certified accuracy at radius $R$ is defined as the fraction of testing examples correctly predicted by the classifier and are provably robust within an $\ell_2$ ball of radius $R$. For brevity, we define the certified accuracies at all radii as ``certified curve".






\begin{figure}[!htbp]
    \centering
    \includegraphics[width=.38\textwidth]{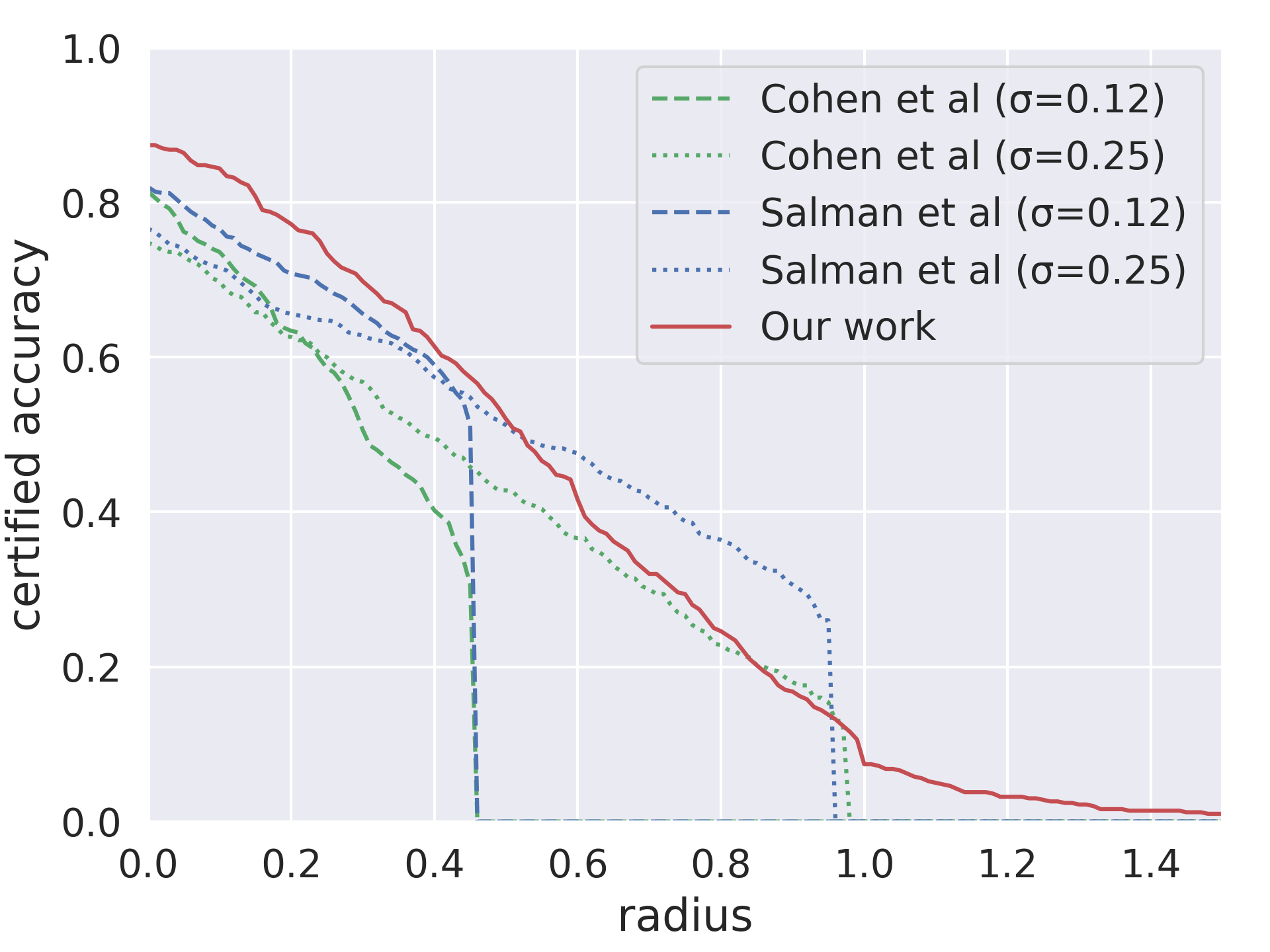}
    \caption{Comparing our adversarially-trained model vs. \cite{cohen2019certified} and \cite{salman2019provably} with $\sigma=0.12$ and $\sigma=0.25$ on CIFAR-10. Our result exhibits the highest natural accuracy under $n = 1$ and $d = 5$. For fair comparison with their methods, we chose the NN models with the top two natural accuracy (the certified accuracy at radius equals to zero).} 
    \label{fig:Our_work_compare}
\end{figure}

\subsection{Experimental Setup}
The architecture of our method is illustrated in Figure~\ref{fig:Arch}. We trained the model by adding spectral normalization to every layer of feature extractor $G$ and reducing the feature vector's dimension to a small number. According to \cite{salman2019provably}, they employ PGD adversarial training to improve the certified curve empirically. Hence, we aslo incorporated it into our model architecture at the expense of decreasing the natural accuracy slightly. As for inference, we applied the reduced feature $d$ to find the exact distance to the decision boundary by using root-finding algorithms from scipy.optimize package \cite{2020SciPy-NMeth}. Details of the algorithm we used will be further discussed in the Appendix.

We employed a 110-layer residual network as our base model for a reasonable comparison, which was also adopted in \cite{cohen2019certified} and \cite{salman2019provably}. Furthermore, there are two primary hyperparameters: the parameter $n$ and the dimension $d$ of the feature vector, affecting our smoothed model.

\subsection{Evaluation Results}
We certified a subset of 500 examples from the CIFAR-10 test set, as done in  \cite{cohen2019certified} and \cite{salman2019provably}, that was released on the Github pages. Our method took 6.48 seconds to certify each example on an NVIDIA RTX 2080 Ti with $n=1$, while 15 seconds are required in \cite{cohen2019certified} under the same setting.

We primarily compare with \cite{cohen2019certified} and \cite{salman2019provably} as they were shown to outperform all other provable $\ell_2$-defenses thoroughly. Recall that our method is norm-independent, as described in Proposition~\ref{prop:1}, we did the experiment in $\ell_2$ norm to compare with their results. As our work claims to maintain high natural accuracy, we utilized $\sigma=0.12$ and $\sigma=0.25$ from \cite{cohen2019certified} and \cite{salman2019provably}, which lead to the top two natural accuracy of their results. For \cite{salman2019provably}, we took one of their results: PGD adversarial training with two steps and $\epsilon=64/255$ for comparison since it has the highest natural accuracy among all of their settings. Note that the results are directly taken from their GitHub pages.

Although we added additional components such as spectral normalization (SN) to the base model for training, the natural accuracy only dropped from 93.7\% to 93.5\%, which nearly carries no harm than adding Gaussian noise for data augmentation. Then, by applying PGD adversarial training with 20 steps and $\|\epsilon\|_2=0.1$ to the base model with SN, the accuracy dropped from 93.5\% to 87.6\%. We call the adversarial trained model ``baseline" in Table~\ref{tab:cifar_attack}. We applied our method with $n = 1$ to the ``baseline" model and got 87.4\% accuracy, which is the highest accuracy among the methods used for comparison. As shown in Figure~\ref{fig:Our_work_compare}, we have the following observations for comparison results: 
\begin{enumerate}
    \item Our method obtains a higher natural accuracy at $R = 0$.
    \item Our method exhibits higher certified accuracies at lower radii ($R < 0.5$) but still obtains non-zero accuracies at larger radii ($R > 1.0$).
\end{enumerate}

To check the rationality of our theoretical results, the empirical upper bounds of the certified curve obtained from practical attacks (PGD attack) were used for comparison. Although there are some uncertainties in root finding, the default of the sum of squares error and the approximate solution error are set to $1.49\times10^{-8}$ which is an acceptable error. Therefore, we can observe from Figure~\ref{fig:Upper_bound} that our certification results approximate the empirical upper bounds. The errors caused by spectral normalization and other possible reasons are discussed in the Appendix.

\begin{figure}[!htbp]
    \centering
    \includegraphics[width=.38\textwidth]{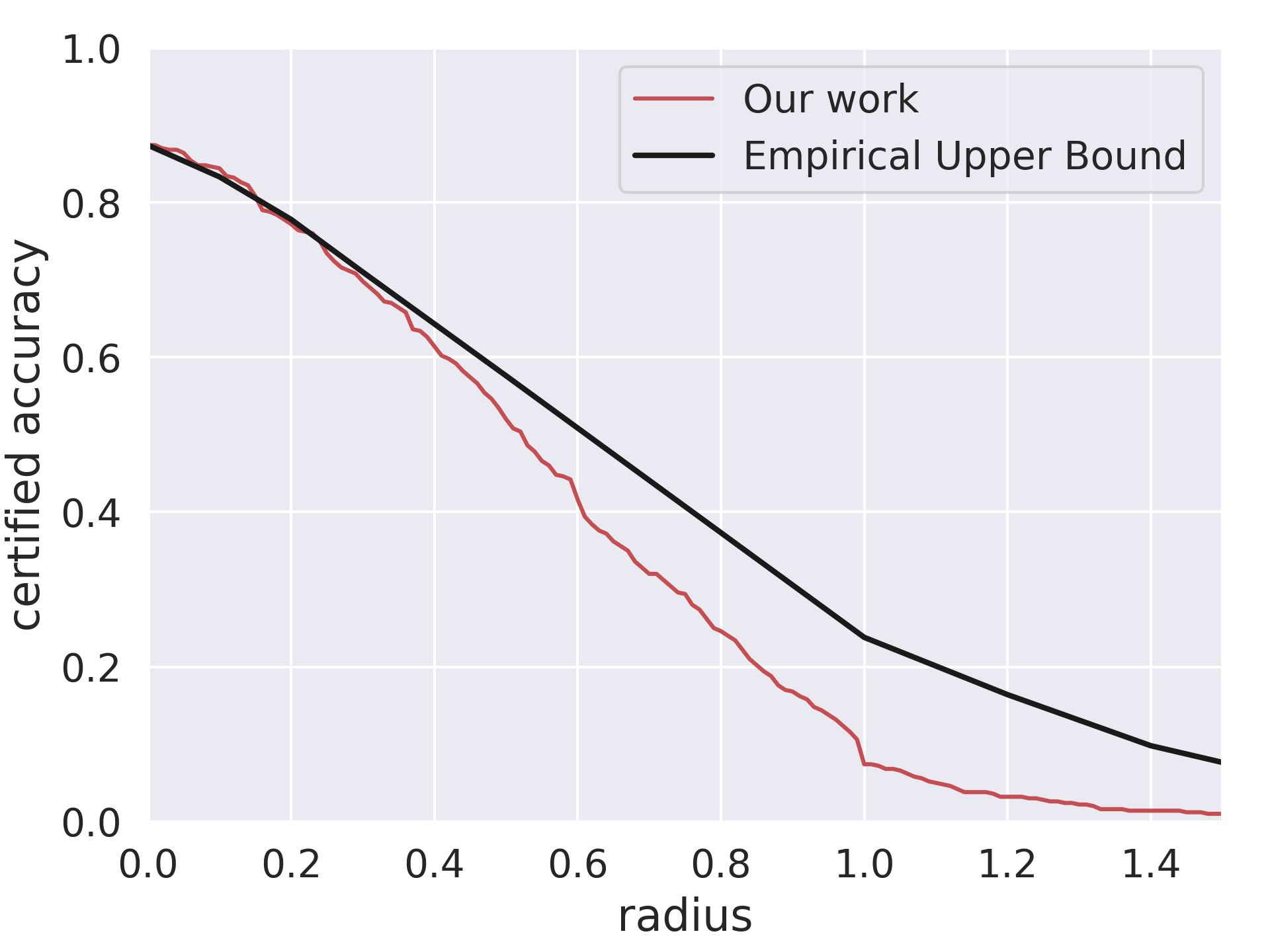}
    \caption{The empirical upper bounds vs. our certified curve on CIFAR-10. The upper bounds are obtained by attacking the smoothed model with $\|\epsilon\|_2\in$  \{0,0.1,0.2,0.3,0.5,1,1.2,1.4,1.5\} via the PGD attack with 20 steps. Our certified curve is relatively conservative, and appears to be beneath the empirical upper bounds.}
    \label{fig:Upper_bound}
\end{figure}

We can observe from Algorithm~\ref{alg:smooth}, where $K$ is the number of the classes of a dataset, $n$ and $d$ are the parameters in Equation~(\ref{eqn:d-dim-bpoly}), the complexity is $O(Kn^d)$. As $K$ increases, the algorithm may cost an unacceptable time. Therefore, due to the high computational complexity, ImageNet \cite{deng2009imagenet} was not considered in our results. The scalability of our work would be an interesting issue for future study.



\subsection{Ablation study}
Our method involves some issues, including dimension reduction, different model architectures, and hyperparameters, that may affect the performance. To explore the effectiveness of each part, we conduct ablation studies. 

\begin{table}[ht]
    \centering
        \begin{tabular}{llll}
            \toprule
            Datasets &   Methods &     Acc &     PGD \\
            \midrule
            \midrule
             \multirow{8}{5em}{CIFAR-10}&
             Baseline &  87.6\%   &  63.6\%\\
             \midrule
             &Ours with $n=1$ &  87.4\% &   63.2\%\\
             &Ours with $n=2$ &  86.8\% &   63.0\%\\
             &Ours with $n=3$ &  87.2\% &   62.6\%\\
             &Ours with $n=4$ &  87.4\% &  62.6\%\\
             &Ours with $n=5$ &  87.4\% &  62.6\% \\
             &Ours with $n=6$ &  87.4\% &   63.0\% \\
             &Ours with $n=7$ &  87.4\% &  63.2\% \\
            \bottomrule
        \end{tabular}
    \caption{Empirical result comparison between the base classifier and the smoothed classifier. The perturbation of the attacks were set to $\|\epsilon\|_\infty=2/255$ in PGD with 20 steps. Acc denotes natural accuracy. Note that CIFAR-10 refers to a subset of 500 examples from the test set.}
    \label{tab:cifar_attack}
\end{table}

First, the robust accuracy should be proportional to the certified accuracy theoretically. We conducted an experiment to find $n$ that results in the best robust accuracy. From Table~\ref{tab:cifar_attack}, we observe that as $n = 1$ and $n = 7$, the natural accuracy and the robust accuracy under PGD are the highest among all $n$s. Since our method with $n = 1$ spends the lowest computation time, we used $n = 1$ in all the other experiments. Remark that our method is a certification instead of a defense method. Thus, as we chose $n = 1 \sim 7$, natural accuracy and robust accuracy are slightly lower than the baseline. However, we discover an intriguing phenomenon in other datasets and models in that we may have better natural accuracy and robust accuracy than the baseline with a proper choice of $n$ (see Table~\ref{tab:mnist_attack} in the Appendix).

\begin{figure}[ht]
    \centering
    \includegraphics[width=.485\textwidth]{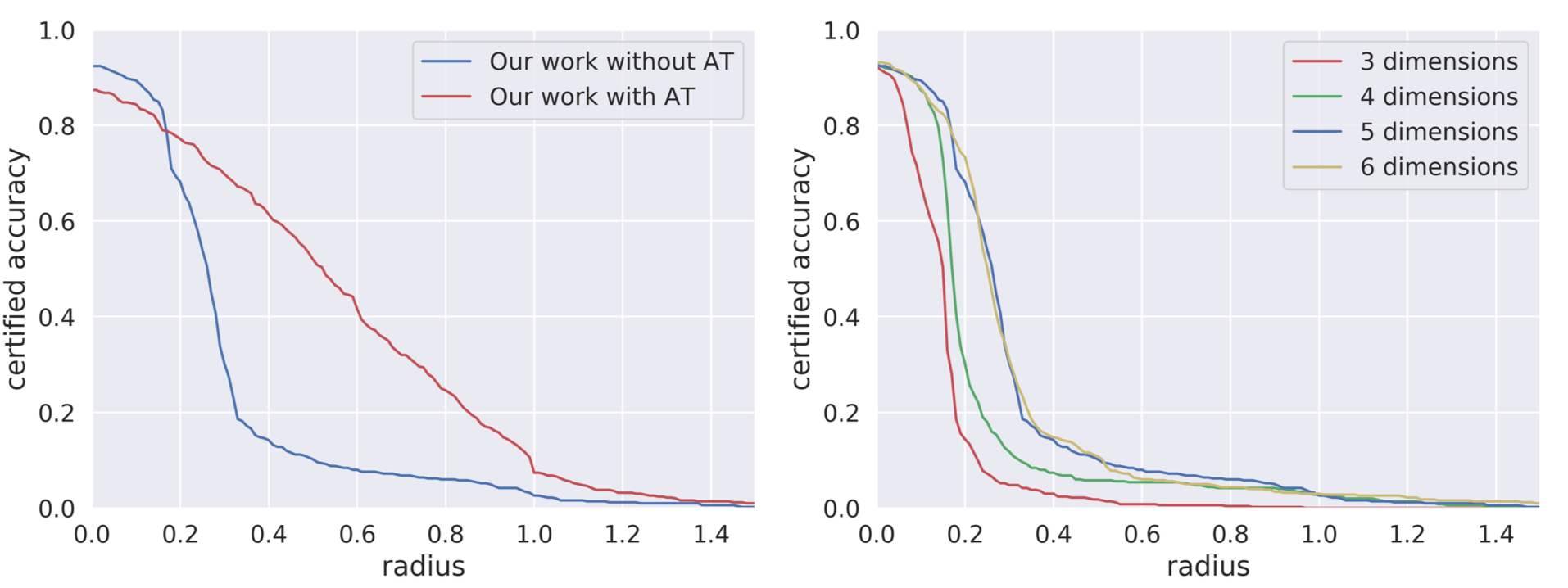}
    \caption{{\bf Left:} Comparison between our method with and without adversarial training. {\bf Right:} Certified curve under different dimensions in our non-adversarially trained smoothed classifier. Note that the dimension represents the number of equations we solve to find the root. }
    \label{fig:merge}
\end{figure}

Second, on the left side of Figure~\ref{fig:merge}, we show that adversarial training is still a useful heuristic strategy that improves the certified curve while sacrificing only limited natural accuracy at radius = $0$. 

In addition, the dimension of the feature vector influences the certified curve a lot. We can observe from the right side of Figure~\ref{fig:merge} that the feature dimensions of length five and six lead to the best certified curve. However, the larger dimension, the longer the computation time. This is why the feature dimension is set to be five in our experiments. Since our feature domain is constrained in the $[0,1]^d$, there is an upper bound of the certified radius, which is the dimension's square root if we choose $\ell_2$ norm. 

Finally, we find that most black-box certification methods only exploit one model in their experiments, which is considered insufficient in some cases. Different model architectures may differ a lot in the certified curve. Hence, we employ several models, including ResNet, WideResNet, VGG, and DenseNet, to support our claim. (See Figure~\ref{fig:different_model} in the Appendix.)

\subsection{Remedy Over-fitting of Regression Model}
In a regression problem, one often has a large amount of features. To prevent over-fitting to the training dataset, feature engineering provides a sophisticated design. Most of them are based on statistical tools to select very few representative data. However, it might also not be able to prevent over-fitting and perform worse with respect to unseen data. Our method provides a post-processing step to smooth the original regression model and preserve its shape. Figure ~\ref{fig:regression} shows an example to verify our claim.  

\begin{figure}[ht]
    \centering
    \includegraphics[width=.33\textwidth]{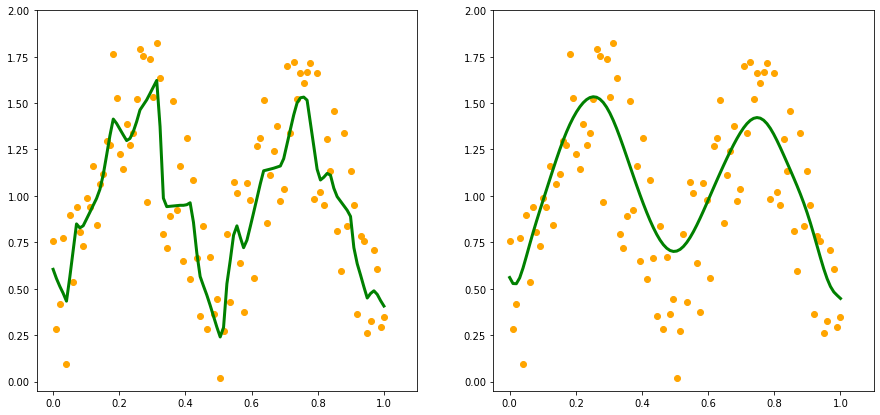}
    \caption{Regression Model. In both graphs, orange dots are noisy training data and green curve denotes the model prediction. {\bf Left:} The original over-fitting regression model. {\bf Right:} After applying our method, the shape is smoothed without over-fitting to outliers.}
    \label{fig:regression}
\end{figure}

\section{Conclusion and Future Work}

We are first to achieve deterministic certification to adversarial attacks via Bernstein polynomial. Moreover, we demonstrate through 2D example visualization, theoretical analyses, and extensive experimentation that our method is a promising research direction for classifier smoothing and robustness certification. Besides smoothing the decision boundary, our method can also be applied to other tasks like over-fitting alleviation. Since the idea ``smoothing" is a general regularization technique, it may be beneficial in various machine learning domains. 

\newpage

\bibliography{main}

\newpage

\onecolumn

\appendix

\section{Appendix of Deterministic Certification to Adversarial Attacks via \\Bernstein Polynomial Approximation}

The appendix is organized as follows. Firstly, we provide a notation table. Secondly, proofs mentioned in the main paper are provided. Thirdly, we offer more experimental results using MNIST, SVHN, and CIFAR-100. Fourthly, we discuss how to solve non-linear systems of equations in detail. Finally, we discuss the certification errors resulted from spectral normalization with skip connections, the singularity of Jacobian matrices, and the conservative constant.

\section{Notations}

\begin{table}[ht]
    \centering
    \begin{tabular}{|p{3.0cm}|p{4.5cm}|}
      \hline
      Notation  &  Definition\\
      \hline
      $I$ & input data\\
      \hline
      $n$ & number of trials\\
      \hline
      $x=[x_1,\cdots,x_d]^T$ & feature vector\\
      \hline
      $x_i$, $i=\{1,2,\cdots,d\}$ & success probability for each trial\\
      \hline
      $f: [0,1]^d \rightarrow \real^K$ & base classifier\\
      \hline
      $\Tilde{f}: [0,1]^d \rightarrow \real^K$ & smoothed classifier\\
      \hline
      $G: \real^{|I|} \rightarrow [0,1]^d$ & feature extractor\\
      \hline
      $y \in \real^K$ & output vector\\
      \hline
      $\Tilde{y} \in \real^K$ & smoothed output vector\\
      \hline
      $B_n(f;x)$ & Bernstein polynomial\\
      \hline
      $C(A;B)$ & A set of continuous functions that maps from A to B\\
      \hline
      $C^\infty(A;B)$ & A set of smooth functions that maps from A to B\\
      \hline
      $[K]=\{1,\cdots,K\}$ & number of classes\\
      \hline
      $\rho: [K] \rightarrow [K]$ & a mapping from ranks to predictions\\
      \hline
      $R$ & certified radius\\
      \hline
      
    \end{tabular}
    \caption{Notations}
    \label{tab:notation}
\end{table}

\section{Proof of Proposition~\ref{prop:1}}

\textbf{Proposition 1. (also stated in the main manuscript)}
\textit{
    Let $\Tilde{f}: [0,1]^d \rightarrow \real^K$ be a smoothed function and let $A=\{\Tilde{y}\in \real^K : \Tilde{y}_i\neq\Tilde{y}_j, \forall i \neq j\}$ denote the area without the smoothed decision boundary. For any $x \in [0,1]^d$ and $\Tilde{f}(x) \in A$, there exists an $R>0$ such that $\argmax_{i\in[K]}\Tilde{f_i}(x+\hat{\delta})= \argmax_{i\in[K]}\Tilde{f_i}(x)$ whenever $\|\hat{\delta}\|_p < R$ \ for \ $p>1$.
}

We can prove a more general case of Proposition~\ref{prop:1} due to the smooth condition is unnecessary as follows:
\begin{lemma}
   Let $f:[0,1]^d \rightarrow \real^K$ be a continuous function, let $A = \{y \in \real^K : y_i \neq y_j \forall i \neq j \} \subset \real^K$, and let $X = \{x \in [0,1]^d : f(x) \in A\} \subset [0,1]^d$. For any $x \in X$, there exists an $R>0$ such that $\argmax_{i\in[K]}f_i(x+\delta)= \argmax_{i\in[K]}f_i(x)$ whenever $\|\delta\|_p < R$ \ for \ $p>1$.
\end{lemma}

\begin{proof}
  The proof contains two cases. The first case states that $\argmax_{i\in[K]}f_i(x+\delta)$ is a set and the second case states that $\argmax_{i\in[K]}f_i(x+\delta) \neq \argmax_{i\in[K]}f_i(x)$.   
  
  {\bf Case 1 : }
  Suppose that there exists a $\|\delta\|_p < R$, and let $x+\delta = x'$ such that $\argmax_{i\in[K]}f_i(x')$ is a set containing multiple elements. Then $f(x') \not\in A \Rightarrow x' \not\in X$.  
  
  {\bf Case 2 : }
  Suppose that there exists a $\|\delta\|_p < R$, and let $x+\delta = x'$ such that $\argmax_{i\in[K]}f_i(x') \neq \argmax_{i\in[K]}f_i(x)$. Define the distance between $f(x)$ and $f(x')$ by $d(f(x),f(x')) = r$, there is an open ball $\ball(f(x),r)$. By the definition of continuous functions in topological spaces, $i.e.$, for each open set in the codomain, its preimage is an open set in that domain. Thus, $f^{-1}(\ball(f(x),r)) = \ball(x,\delta)$ should be an open subset of $X$. However, there exists some $\delta_1 < \delta$ such that $x+\delta_1 \not\in X$. This indicates that $f$ is not continuous on $X$. This is a contradiction.
  
  By the results of {\bf Case 1} and {\bf Case 2}, the lemma follows.
 
\end{proof}

Hence, our method can also be applied to neural networks directly. However, due to the massive amounts of deep neural network parameters, the result may be inaccurate. Besides, as $n$ grows larger or the feature vector's dimension extends, the certification process will be intractable. Thus, we only solve the problem with small $n$ and small dimensions of the feature vector.

\section{Additional Experiments}

The datasets we use are summarized in Table~\ref{table:datasets}. Among them, MNIST and CIFAR-10 have been popularly used in this area. Since we do not have the certified results of other methods in other datasets, we only conduct experimental comparison between the base classifier and the smoothed classifier to explore the effect of trials empirically. 
\begin{table}[!htbp]
    \centering
          \begin{tabular}{|c|c|c|c|c|c|c|}
          \hline
          \textbf{Dataset} & \textbf{Image size} & \textbf{Training set size}
            & \textbf{Testing set size} & \textbf{Classes}\\
          \hline
          {\bf CIFAR-100} & 32x32x3 & 50K &10K  & 100 \\
          \hline
          {\bf CIFAR-10} & 32x32x3 & 50K &10K  & 10 \\
          \hline
          {\bf SVHN}  & 32x32x3 & 73K &26K  & 10 \\
          \hline
         {\bf MNIST} & 28x28x1 & 60K &10K  & 10 \\
          \hline
        \end{tabular}
      \caption{The evaluation datasets are sorted based on descending order of image size.}
    \label{table:datasets}
\end{table}

\subsection{For CIFAR-10}
As mentioned before, different model architecture may differ a lot in the certified curve. In Figure~\ref{fig:different_model}, we can see that the certified curve varies a lot. For certified radius at 0, the accuracy differs from 89\% (Wide ResNet) to 81\% (VGG19). As for radius at 0.2, the accuracy also varies in a wide range. Although the tendency of certified curves look similar, their gaps may differ by more than 10\%, which is a significant difference.

\begin{figure}[!htbp]
    \centering
    \includegraphics[width=.45\textwidth]{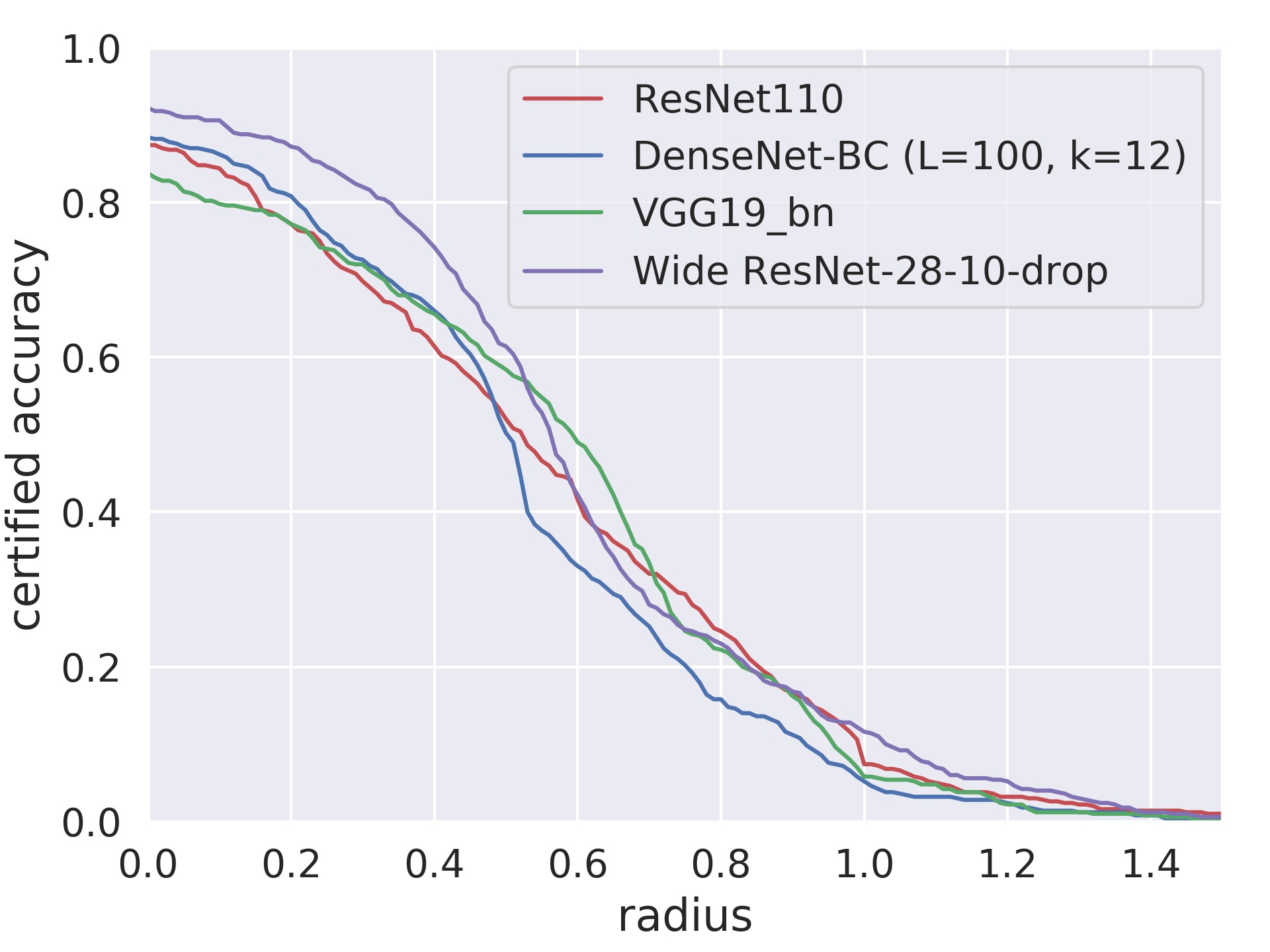}
    \caption{Employing different adversarially trained models to show the certified curve is actually model-dependent.}
    \label{fig:different_model}
\end{figure}

\subsection{For MNIST}
We include MNIST \cite{lecun2010mnist} to demonstrate some of our claims in a simple manner. As we mentioned in Table~\ref{tab:mnist_attack} of the main paper, if we take $n=5$, the natural accuracy 98.51\% and robust accuracy 59.16\% in our method are both larger than the baseline with the natural accuracy 98.50\% and robust accuracy 52.25\%. As we can also see in Figure~\ref{fig:mnist_diff} that there is a trade-off between the number of $n$ and certified accuracy.

\begin{table}[!htbp]
    \centering
        \begin{tabular}{lllll}
            \toprule
            Datasets &   Methods &     Acc &    FGSM &     PGD\\
            \midrule
            \midrule
             \multirow{8}{3em}{MNIST}&
             Baseline &  98.50\% &  61.46\% &  52.25\% \\
             \midrule
             &Ours with $n=1$ &  57.09\% &  43.79\% &  42.63\% \\
             &Ours with $n=2$ &  77.51\% &  48.67\% &  45.99\% \\
             &Ours with $n=3$ &  98.29\% &  64.37\% &  61.33\% \\
             &Ours with $n=4$ &  98.45\% &  64.55\% &  60.67\% \\
             &Ours with $n=5$ &  98.51\% &  63.26\% &  59.16\% \\
             &Ours with $n=6$ &  98.51\% &  62.82\% &  58.06\% \\
             &Ours with $n=7$ &  98.50\% &  61.76\% &  57.03\% \\
            \bottomrule
        \end{tabular}
    \caption{({\bf MNIST}) Empirical result comparison between the base classifier and the smoothed classifier. Note that $\|\epsilon\|_\infty=0.1$ was set for attacks in FGSM and PGD (20 steps). Acc denotes natural accuracy.}
    \label{tab:mnist_attack}
\end{table}

\begin{figure}[!htbp]
    \centering
    \includegraphics[width=.45\textwidth]{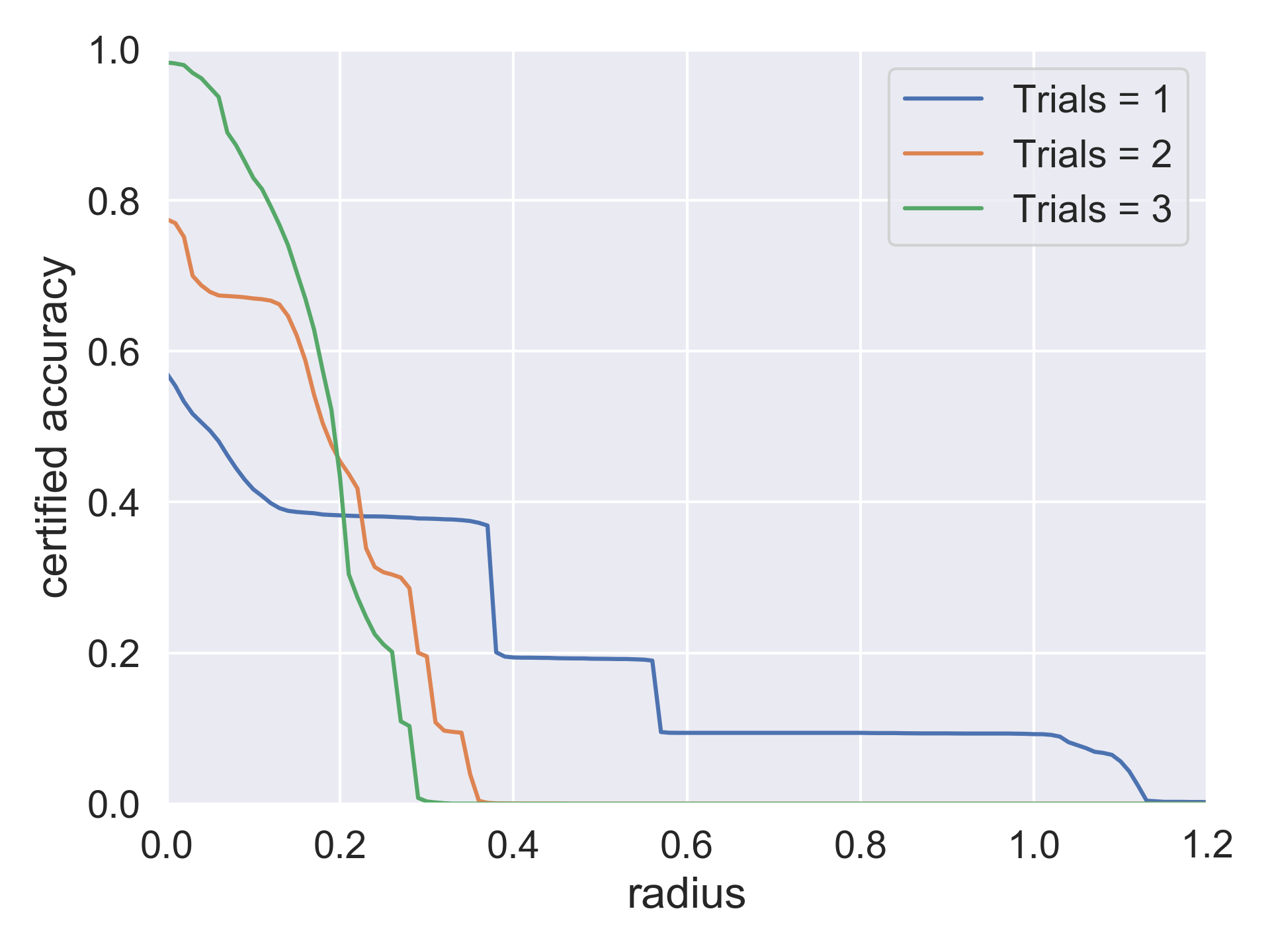}
    \caption{Trade-off between the number of $n$ and certified accuracy (MNIST is taken as an example here. We see similar phenomena in other datasets.) Smaller/larger $n$s exhibit larger/smaller certified radii but suffer from lower/higher natural accuracies.}
    \label{fig:mnist_diff}
\end{figure}


\begin{table}[!htbp]
    \centering
        \begin{tabular}{lllll}
            \toprule
            Datasets &   Methods &     Acc &    FGSM &     PGD \\
            \midrule
            \midrule
             \multirow{8}{3em}{SVHN}&
             Baseline &  95.07\% &  56.85\% &  38.66\% \\
             \midrule
             &Ours with $n=1$ &  78.31\% &  47.80\% &  34.60\% \\
             &Ours with $n=2$ &  94.03\% &  55.80\% &  39.13\% \\
             &Ours with $n=3$ &  94.59\% &  56.63\% &  39.55\% \\
             &Ours with $n=4$ &  94.71\% &  57.00\% &  39.42\% \\
             &Ours with $n=5$ &  94.82\% &  57.28\% &  39.40\% \\
             &Ours with $n=6$ &  94.88\% &  57.39\% &  39.49\% \\
             &Ours with $n=7$ &  94.91\% &  57.54\% &  39.48\% \\
            \bottomrule
        \end{tabular}
    \caption{({\bf SVHN}) Empirical result comparison between the base classifier and the smoothed classifier. Note that $\|\epsilon\|_\infty=0.01$ was set for attacks in FGSM and PGD with 20 steps. Acc denotes natural accuracy.}
    \label{tab:svhn_attack}
\end{table}


\begin{table}[!htbp]
    \centering
        \begin{tabular}{lllll}
            \toprule
            Datasets &   Methods &     Acc &    FGSM &     PGD \\
            \midrule
            \midrule
             \multirow{8}{5em}{CIFAR100}&
             Baseline &  61.6\% & 33.6\% &  38.4\% \\
             \midrule
             &Ours with $n=1$ &  33.8\% &  16.2\% &  17.4\% \\
             &Ours with $n=2$ &  52.8\% &  27.8\% &  31.0\% \\
             &Ours with $n=3$ &  59.0\% &  32.4\% &  36.4\% \\
             &Ours with $n=4$ &  59.4\% &  33.4\% &  38.4\% \\
             &Ours with $n=5$ &  60.6\% &  34.2\% &  38.6\% \\
             &Ours with $n=6$ &  60.0\% &  33.8\% &  38.2\% \\
            \bottomrule
        \end{tabular}
    \caption{({\bf CIFAR-100}) Empirical result comparison between the base classifier and the smoothed classifier. Note that $\|\epsilon\|_\infty=2/255$ was set for attacks in FGSM and PGD with 20 steps. Acc denotes natural accuracy.}
    \label{tab:cifar100_attack}
\end{table}

\begin{figure}[!htbp]
    \centering
    \includegraphics[width=.45\textwidth]{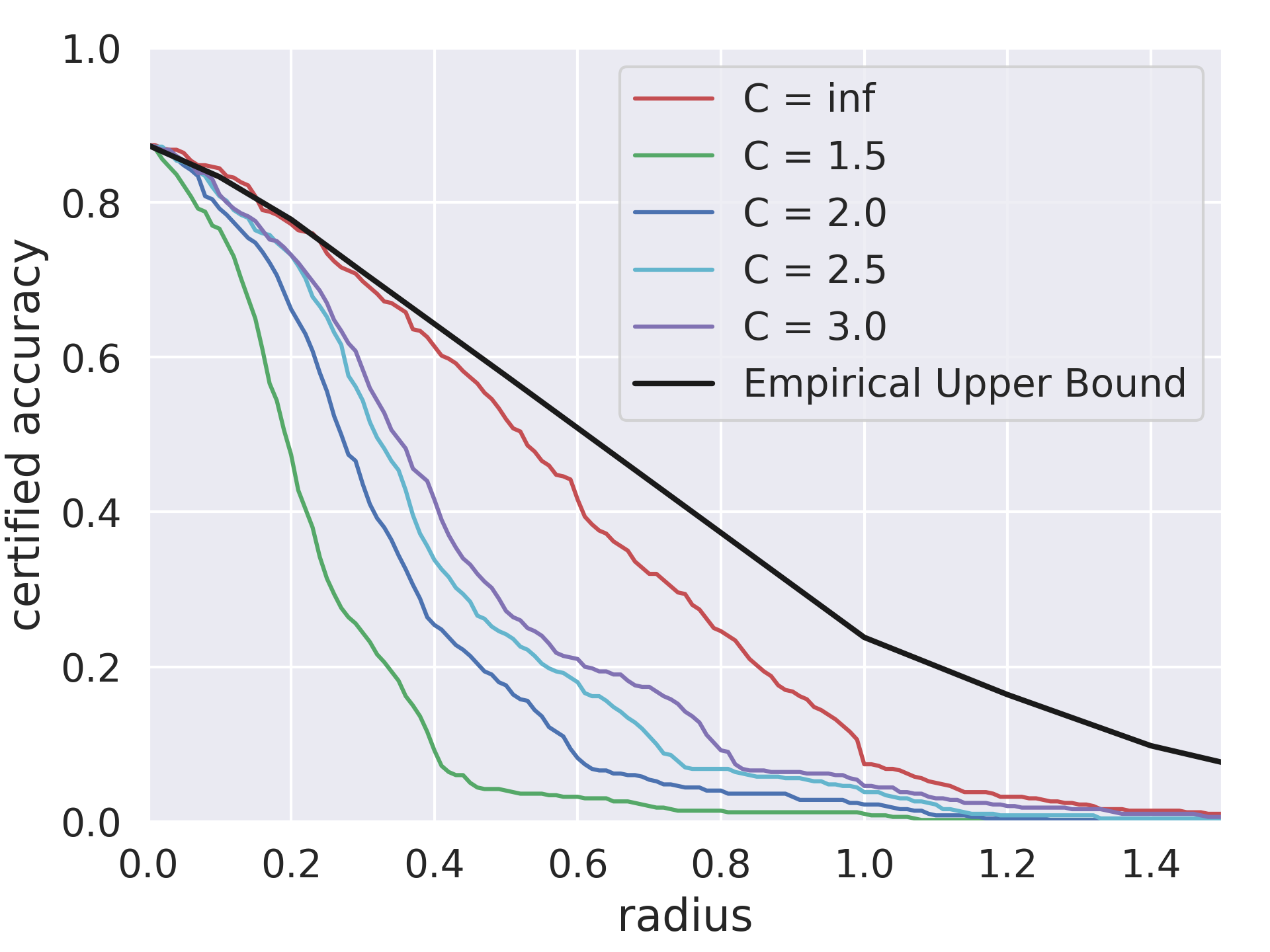}
    \caption{Comparison with different Cs. As C is a conservative parameter, the larger the C is, the closer it is towards the empirical upper bounds.}
    \label{fig:different_c}
\end{figure}

\clearpage
\section{Non-Linear Least Squares Problems and Root-Finding}

In Algorithm~\ref{alg:cert} of the main paper, we adopt a least-square solver. \cite{madsen2004methods} is a great resource for solving nonlinear least squares problem so we suggest readers to find more details in this book. To make this paper self-contained, we describe some important methods from \cite{madsen2004methods} below.  

First we consider a function $f: \real^n \rightarrow \real^m$ with $m \geq n$. Minimizing $\|f(x)\|$ is equivalent to finding 
\begin{align*}
    x^* &= \argmin_{x}\{F(x)\}.\\
    F(x) &= \frac{1}{2}\sum\limits_{i=1}^{m}(f_i(x))^2 = \frac{1}{2}\|f(x)\|^2 = \frac{1}{2}f(x)^Tf(x).
\end{align*}
Let $f \in C^2(\real^n;\real^m)$, and consider Taylor Expansion
\begin{align}
    f(x+h) = f(x) + J(x)h + O(\|h\|^2),
\end{align}
where $J(x) \in \real^{m \times n}$ is the Jacobian matrix, we have $(J(x))_{ij} = \frac{\partial f_i}{\partial x_j}(x)$. Since $F:\real^n \rightarrow \real$, we can denote the first derivative and second derivative of $F$, respectively, in Equations (12) (13) and Equations (14) (15) as:
\begin{align}
    \frac{\partial F}{\partial x_j} &= \sum\limits_{i=1}^mf_i(x)\frac{\partial f_i(x)}{\partial x_j}.\\
    F'(x) &= J(x)^Tf(x). \label{eqn:Fprime}\\
    \frac{\partial^2F}{\partial x_j \partial x_k} &= \sum\limits_{i=1}^m \frac{\partial f_i(x)}{\partial x_j}\frac{\partial f_i(x)}{\partial x_k} + f_i(x)\frac{\partial f_i(x)}{\partial x_j \partial x_k}.\\
    F''(x) &= J(x)^TJ(x) + \sum\limits_{i=1}^m f_i(x)f_i''(x).
\end{align}

\subsection{Gradient Descent (Newton-Raphson's method)}
Let $f(x+h) \approx l(h) := f(x) + J(x)h$, which is the first order Taylor expansion. Suppose that we find an $h$ such that $f(x+h) = 0$. We have the following iteration form:
\begin{align}
    J(x_k)h &= -f(x_k)\\
    x_{k+1} &= x_k + h.
\end{align}
\subsection{Gauss-Newton method}
First, we also consider $f(x+h) \approx l(h) := f(x) + J(x)h$ and $F(x+h) \approx L(h)$, where
\begin{align}
    L(h) &:= \frac{1}{2}l(h)^Tl(h),  \hspace{5ex}(with \ f=f(x) \ and \ J=J(x)) \notag\\
         &= \frac{1}{2} [(f^T+h^TJ^T)(f+Jh)] \notag\\
         &= \frac{1}{2} [f^Tf + f^TJh + h^TJ^Tf + h^TJ^TJh]\notag\\
         &= \frac{1}{2}f^Tf + h^TJ^Tf + \frac{1}{2}h^TJ^TJh\notag\\
         &= F(x) + h^TJ^Tf + \frac{1}{2}h^TJ^TJh.
\end{align}
We can further compute the first and second derivative of $L(h)$, respectively, as
\begin{align}
    L'(h) &= J^Tf + J^TJh \label{eqn:first-order}\\
    L''(h) &= J^TJ. \label{eqn:second-order}
\end{align}
From Equations~(\ref{eqn:Fprime}) and (\ref{eqn:first-order}), we have $L'(0) = F'(x)$. If $J$ has full rank, $L''(h)$ is positive definite. We have the following iteration form:
\begin{align}
    J^T(x_k)J(x_k)h_{gn} &= -J^T(x_k)f(x_k) \label{eqn:GN}\\
    x_{k+1} &= x_k + \alpha h_{gn}.
\end{align}
Note that $h_{gn} = \argmin\limits_h\{L(h)\}$ is called Gauss-Newton step and it is a descent direction since
\begin{align}
    h_{gn}^TF'(x) =  h_{gn}^T(J^Tf) = -h_{gn}^TJ^TJh_{gn} < 0. \label{eqn:steep}
\end{align}
Provided that 
\begin{enumerate}
    \item $\{x | F(x) \leq F(x_0)\}$ is bounded and
    \item the Jacobian $J(x)$ has full rank in all steps,
\end{enumerate}
Gauss-Newton method has a convergence guarantee. To overcome the constraint that the Jacobian should have full rank in all steps, the Levenberg–Marquardt method adds a damping parameter $\mu$ to modify Equation~(\ref{eqn:GN}).

\subsection{The Levenberg–Marquardt (LM) method}
\cite{levenberg1944method} and \cite{marquardt1963algorithm} defined the step $h_{lm}$ by modifying Equation~(\ref{eqn:GN}) as:
\begin{align}\label{eqn:lm_form}
    (J^T(x_k)J(x_k) + \mu I)h_{lm} &= -J^T(x_k)f(x_k).
\end{align}
The above modification has lots of benefits; we list some of them as follows:
\begin{enumerate}
    \item $\forall \mu > 0$, $(J^TJ + \mu I)$ is positive definite which can be proved by the definition of positive definite. This ensures $h_{lm}$ is a descent direction; the proof is similar to Equation~(\ref{eqn:steep}).
    \item For large values of $\mu$, we get $h_{lm} \approx -\frac{1}{\mu}J^Tf = -\frac{1}{\mu}F'(x)$. This is good when $x_k$ is far from the solution.
    \item If $\mu$ is very small, then $h_{lm} \approx h_{gn}$. This is also a good direction if x is close to the solution $x^*$.
\end{enumerate}
We employ {\bf Levenberg–Marquardt method (``LM")} \cite{levenberg1944method} \cite{marquardt1963algorithm} for $\ell_2$ norm certification. The LM method, which is consumed with the advantages of both the {\bf Gauss-Newton} method and {\bf Gradient Descent}, is commonly adopted to solve Equation~(\ref{eqn:nonlinear}). 
\subsection{Trust region method}

Recall from Equation~(\ref{eqn:nonlinear}), we can also add a conservative constant $\xi$ which is similar to the idea of soft margin in the support vector machine. We abbreviate the notation by setting $\beta_i = B_n(f_{i};x)$ and $S(\beta)_i = \exp(\beta_i)/\sum_{j = 1}^{K}\exp(\beta_j)$, where $K$ is the number of classes. Note that $S(\beta)$ is the traditional softmax function which represents the probability vector. Also, recall that $d \leq K$ is the dimension of a feature vector $x_0$.

In the worst case analysis, the runner-up class is the easiest target for attackers. Hence, we assume that the closest point to a feature vector $x_0$ on the decision boundary between the top two predictions of the smoothed classifier $\Tilde{f}(x) = \beta = [\beta_1,\cdots,\beta_K]$ follows $\beta_{\rho(1)} = \beta_{\rho(2)}$, where $\rho$ is the mapping from ranks to predictions. Therefore, the nonlinear systems of equations for characterizing the points on the decision boundary of $\Tilde{f}(x)$ are described as:

\begin{equation}
    \begin{cases}
    \phi_0 := \beta_{\rho(1)} - \beta_{\rho(2)} - \xi = 0\\
    \phi_1 := S(\beta)_{\rho(1)} - \frac{1}{1+\exp(-\xi)} = 0\\
    \phi_2 := S(\beta)_{\rho(2)} - \frac{1}{1+\exp(\xi)} = 0\\
    \phi_3 := S(\beta)_{\rho(3)} = 0\\
    \hspace{8ex}\vdots\\
    \phi_d := S(\beta)_{\rho(d)} = 0\\ 
    \end{cases}
    \label{eqn:nonlinear2}
\end{equation}
where $\xi = \frac{\beta_{\rho(1)} - \beta_{\rho(2)}}{C}$ is a conservative constant and C is a parameter (see Figure~\ref{fig:different_c}). In addition, $\phi_0$ can be seen as the decision boundary between the top two predictions if $C \rightarrow \infty$. By $\phi_0$, we know that $\beta_{\rho(2)} = \beta_{\rho(1)} - \xi$. We can omit $\beta_{\rho(3)},\cdots,\beta_{\rho(K)}$ since they are too small. By employing softmax function, we have $S(\beta)_{\rho(1)} = \frac{\exp(\beta_{\rho(1)})}{\exp(\beta_{\rho(1)}) + \exp(\beta_{\rho(2)})} = \frac{\exp(\beta_{\rho(1)})}{\exp(\beta_{\rho(1)}) + \exp(\beta_{\rho(1)} - \xi)} = \frac{1}{1+\exp(-\xi)}$.
Also, $S(\beta)_{\rho(2)}$ can be computed in the same way as $S(\beta)_{\rho(1)}$ by setting $\beta_{\rho(1)} = \beta_{\rho(2)} + \xi$. We can observe from $S(\beta)_{\rho(1)}$ and $S(\beta)_{\rho(2)}$ that the probability of top two predictions has a controllable gap $\frac{1}{1+\exp(-\xi)} - \frac{1}{1+\exp(\xi)}$. The remaining equations indicate that the probabilities other than $\rho(1)$ and $\rho(2)$ are all zeros. Equation~(\ref{eqn:nonlinear2}) reveals the decision boundary of the smoothed classifier if $\xi$ is zero, otherwise it represents a margin slightly away from the decision boundary. The only thing we need to do is to start from the initial guess (feature vector $x_0$) and find the adversary nearest to it.



We can replace Equation~(\ref{eqn:nonlinear}) with Equation~(\ref{eqn:nonlinear2}) and solve Equation~(\ref{eqn:ls}) by means of the trust region method \cite{nocedal2006numerical}. First we consider the model function $m_k(s_k)$ by first order Taylor expansion of $\Phi(x_k + s_k)$, that is, $F(x_k + s_k) \approx m_k(s_k) = \frac{1}{2}\|\Phi(x_k) + J_ks_k\|_2^2$. To be precise, $m_k(s_k)$ can be written as:
\begin{align}
    m_k(s_k) &= F(x_k) + s_k^TJ_k^T\Phi(x_k) + \frac{1}{2}s_k^TB_ks_k\\
    \mathcal{T}_k &= \{x \in \real^n : \|x - x_k\|_p  \leq \Delta_k \},
\end{align}
where $J_k = \nabla \Phi(x_k)$, $B_k = J_k^TJ_k$ is an approximation of $\nabla^2 \Phi(x_k)$, $\mathcal{T}_k$ is called the trust region, and $\Delta_k$ is the radius of the trust region.

Besides, we need to define a step ratio $r_k$ by 
\begin{align}\label{eqn:ratio}
    r_k &= \frac{F(x_k) - F(x_k+s_k)}{m_k(0) - m_k(s_k)} \notag \\
        &= \frac{\|\Phi(x_k)\|_2^2 - \|\Phi(x_k+s_k)\|_2^2}{\|\Phi(x_k)\|_2^2 - \|\Phi(x_k) + J_ks_k\|_2^2},
\end{align}
where we call the numerator actual reduction and the denominator predicted reduction. $r_k$ is used to determine the acceptance or not of the step $s_k$ and update (expand or contract) $\Delta_k$. Note that the step $s_k$ is obtained by minimizing the model $m_k$ over the region,  including 0. Hence, the predicted reduction is always nonnegative. If $r_k$ is negative, then $F(x_k+s_k)$ is greater than $F(x_k)$ and $s_k$ must be rejected. On the other hand, if $r_k$ is large, it is safe to expand the trust region.

The trust region method \cite{nocedal2006numerical} is depicted in Algorithm~\ref{alg:trust}. For each iteration, the step $s$ is computed by solving an approximate solution of the following subproblem,  
\begin{align}\label{eqn:subproblem}
   \min_{s\in\real^n}m_k(s), \hspace{5pt} \text{subject to} \hspace{2pt} \|s\|_p &\leq \Delta_k,
\end{align}
where $\Delta_k$ is the radius of the trust region.

\begin{algorithm}
\SetAlgoLined
\KwInput{$\hat{\Delta} > 0$, $\Delta_0 \in (0, \hat{\Delta})$ and $\eta \in [0,\frac{1}{4})$}
\KwOutput{Solution $x_k$}
 \For{$k = 0,1,2\cdots$}{
  Get $s_k$ by approximately solving Equation~(\ref{eqn:subproblem})\;
  Evaluate $r_k$ from Equation~(\ref{eqn:ratio})\;
  \eIf{$r_k < \frac{1}{4}$}{
   $\Delta_{k+1} = \frac{1}{4}\Delta_k$\;
   }{
     \eIf{$r_k > \frac{3}{4}$ and $\|s_k\| = \Delta_k$}
     {$\Delta_{k+1} = \min(2\Delta_k,\hat{\Delta})$\;}
     {$\Delta_{k+1} = \Delta_k$\;}
  }
  \eIf{$r_k > \eta$}
  {$x_{k+1} = x_k + s_k$\;}
  {$x_{k+1} = x_k$\;}
 }
 \caption{Trust region method \cite{nocedal2006numerical}}
 \label{alg:trust}
\end{algorithm}
Due to different $\ell_p$ norm, we should change our approach to solve Equation~(\ref{eqn:subproblem}). As mentioned in the main paper, we choose LM to solve Equation~(\ref{eqn:subproblem}) since the trust region is based on $\ell_2$ norm. If we solve Equation~(\ref{eqn:subproblem}) in $\ell_\infty$ norm, we will adopt the dogbox method \cite{voglis2004rectangular}, which is a modification of the traditional dogleg trust region method \cite{nocedal2006numerical}. As for other norms, we can either use the traditional dogleg method with different trust regions or solve the minimization problem by other approaches mentioned in scipy.optimize package.


In the 2D example (see Figure~\ref{fig:compare}), we consider a simpler version of Equation~(\ref{eqn:nonlinear}), which is $\phi_0$,  and F(x) becomes $\frac{1}{2}\phi_0(x)^2$. Recall that our goal is to find a point on the decision boundary which is the closest point to the feature $x_0$. Thus, we define the objective function as $\hat{F}(x) = \frac{1}{2}\phi_0(x)^2 + \|x_0 - x\|^2_2$, where $\|x_0 - x\|^2_2$ is a regularization term, and adopt the trust region method \cite{nocedal2006numerical} to minimize $\hat{F}(x)$. Minimizing $\hat{F}(x)$ is a generalized optimization problem with arbitrary dimensions of $x$. Thus, $-\hat{F}(x)$ can be added to the loss function as a regularization term to fine-tune the model.

The certification procedure (Algorithm~\ref{alg:cert}) is time-consuming as $n$ becomes large. To solve this problem, we can adopt the neural network approach. As \cite{agrawal2019differentiable} stated, the convex optimization problem is unfolded into a neural network architecture. Hence, we can use the fashion of neural networks to accelerate the certification procedure, which is an interesting future work.


\section{The uncertainty of root-finding}
Unlike most model architectures contain skip connections, in spectral normalization, we do not consider skip connections. If the feature extractor $G$ contains a skip connection, the Lipschitz constant should be computed by
\begin{align*}
    \|G(x) + x - G(x') - x'\|_2 &< \|G(x) - G(x')\|_2 + \|x-x'\|_2\\
                              &< 2\|x-x'\|_2.
\end{align*}
If more skip connections are considered, the Lipschitz constant will grow in the power of two. This is a reason why spectral normalization will cause the error. In the training procedure, there might be other ways to control the Lipschitz constant. It is an interesting future work to discuss how to control the Lipschitz constant with skip connections. 

Consider Equation~(\ref{eqn:nonlinear}), we can see that the Jacobian is singular, which is a bad condition for solving minimization problem. In this case, the solution might be inaccurate. However, in practice, we have an acceptable result (see Figure~\ref{fig:Upper_bound}) due to the modification of LM (see Equation~\ref{eqn:lm_form}). A better description of the decision boundary is needed to make the Jacobian non-singular which is also an interesting topic to discuss in the future.

The other uncertainty is due to the conservative parameter $\xi = \frac{\beta_{\rho(1)} - \beta_{\rho(2)}}{C}$ in Equation~(\ref{eqn:nonlinear2}). We can observe from Figure~\ref{fig:different_c} that as $C$ gradually grows, the certified curve will approximate the empirical upper bounds. However, in some scenarios, we can only choose small $C$ to make the certified result conservative enough.

\end{document}